\documentclass{article}

\usepackage{times}
\usepackage{graphicx}
\usepackage{subfigure} 
\usepackage{natbib}
\usepackage{algorithm}
\usepackage{algorithmic}
\usepackage{hyperref}

\usepackage[accepted]{icml2017} 
\usepackage{amsmath}
\usepackage{amsthm}
\usepackage{amssymb}
\usepackage{url}
\usepackage{amsbsy}
\usepackage{upgreek}
\usepackage[mathscr]{euscript}
\usepackage[whileod]{algochl}
\usepackage{gastex}

\DeclareMathOperator*{\E}{\mathbb E}
\DeclareMathOperator*{\argmax}{\rm argmax}
\DeclareMathOperator*{\argmin}{\rm argmin}

\DeclareMathOperator{\sgn}{sgn}
\DeclareMathOperator{\conv}{\rm conv}
\DeclareMathOperator{\Margin}{\rm Margin}
\DeclareMathOperator{\prox}{\rm prox}

\newcommand{\Rset}{\mathbb R}

\newcommand{\h}{\widehat}
\newcommand{\tl}{\widetilde}
\newcommand{\ov}{\overline}

\newcommand{\R}{\mathfrak R}
\newcommand{\sX}{{\mathscr X}}

\newcommand{\sD}{{\mathscr D}}
\newcommand{\sF}{{\mathscr F}}
\newcommand{\sG}{{\mathscr G}}
\newcommand{\sH}{{\mathscr H}}
\newcommand{\bE}{\mathbb{E}}
\newcommand{\cH}{\mathcal H}
\newcommand{\cO}{\mathcal O}
\newcommand{\cX}{\mathcal X}

\newcommand{\be}{\boldsymbol \epsilon}

\newcommand{\bz}{\mathbf z}
\newcommand{\bw}{\mathbf w}

\newcommand{\bu}{\mathbf u}
\newcommand{\bh}{\mathbf h}
\newcommand{\bzero}{\mathbf 0}

\newcommand{\bPsi}{\boldsymbol \Psi}

\newcommand{\ssigma}{{\boldsymbol \sigma}}
\newcommand{\EQ}{\gets}
\newcommand{\TO}{\mbox{ {\bf to }}}

\newcommand{\eps}{\epsilon}
\newcommand{\e}{\epsilon}
\newcommand{\set}[2][]{#1 \{ #2 #1 \} }
\newcommand{\ignore}[1]{}

\newtheorem*{rep@theorem}{\rep@title}
\newcommand{\newreptheorem}[2]{%
\newenvironment{rep#1}[1]{%
 \def\rep@title{#2 \ref{##1}}%
 \begin{rep@theorem}}%
 {\end{rep@theorem}}}
\makeatother

\newtheorem{lemma}{Lemma}
\newtheorem{theorem}{Theorem}

\newtheorem{corollary}{Corollary}

\newreptheorem{theorem}{Theorem}
\newreptheorem{lemma}{Lemma}
\newreptheorem{corollary}{Corollary}
\newreptheorem{proposition}{Proposition}

\hypersetup{
  colorlinks   = true,
  urlcolor     = blue,
  linkcolor    = blue,
  citecolor    = blue
}

\icmltitlerunning{AdaNet: Adaptive Structural Learning of Artificial Neural Networks}

\begin{document} 

\twocolumn[
\icmltitle{AdaNet: Adaptive Structural Learning of Artificial Neural Networks}

\begin{icmlauthorlist}
\icmlauthor{Corinna Cortes}{google}
\icmlauthor{Xavier Gonzalvo}{google}
\icmlauthor{Vitaly Kuznetsov}{google}
\icmlauthor{Mehryar Mohri}{cims,google}
\icmlauthor{Scott Yang}{cims}
\end{icmlauthorlist}

\icmlaffiliation{cims}{Courant Institute, New York, NY, USA}
\icmlaffiliation{google}{Google Research, New York, NY, USA}

\icmlcorrespondingauthor{Vitaly Kuznetsov}{vitalyk@google.com}

\icmlkeywords{Neural Networks, Ensemble Methods, Learning Theory}

\vskip 0.3in
]

\printAffiliationsAndNotice{}

\begin{abstract} 
  We present new algorithms for adaptively learning artificial neural
  networks.  Our algorithms (\textsc{AdaNet}) adaptively learn both
  the structure of the network and its weights. They are based on a
  solid theoretical analysis, including data-dependent generalization
  guarantees that we prove and discuss in detail. We report the
  results of large-scale experiments with one of our algorithms on
  several binary classification tasks extracted from the CIFAR-10
  dataset. The results demonstrate that our algorithm can
  automatically learn network structures with very competitive
  performance accuracies when compared with those achieved for neural
  networks found by standard approaches.

\ignore{
  We present a new framework for analyzing and learning artificial
  neural networks. Our approach simultaneously and adaptively learns
  both the structure of the network as well as its weights. The
  methodology is based upon and accompanied by strong data-dependent
  theoretical learning guarantees, so that the final network
  architecture provably adapts to the complexity of any given
  problem. We validate our framework with experimental data on
  real-life dataset which demonstrate that our algorithm can
  automatically find network structures with accuracies competitive to
  networks found by traditional approaches.
}
\end{abstract} 

\section{Introduction}
\label{sec:intro}

Deep neural networks form a powerful framework for machine learning
and have achieved a remarkable performance in several areas in recent
years.  Representing the input through increasingly more abstract
layers of feature representation has shown to be extremely effective
in areas such as natural language processing, image captioning, speech
recognition and many others
\citep{KrizhevskySutskeverHinton2012,SutskeverVinyalsLe2014}.\ignore{
  The concept of multilayer feature representations and modeling
  machine learning problems using a network of neurons is also
  motivated and guided by studies of the brain, neurological behavior,
  and cognition.}  However, despite the compelling arguments for using
neural networks as a general template for solving machine learning
problems, training these models and designing the right network for a
given task has been filled with many theoretical gaps and practical
concerns.

To train a neural network, one needs to specify the parameters of a
typically large network architecture with several layers and units,
and then solve a difficult non-convex optimization problem. From an
optimization perspective, there is no guarantee of optimality for a
model obtained in this way, and often, one needs to implement ad hoc
methods (e.g. gradient clipping or batch normalization
\citep{PascanuMikolovBengio2013,IoffeSzegedy15}) to derive a coherent
models.

Moreover, if a network architecture is specified a priori and trained
using back-propagation, the model will always have as many layers as
the one specified because there needs to be at least one path through
the network in order for the hypothesis to be non-trivial. While
single weights may be pruned \citep{HanPoolTranDally}, a technique
originally termed Optimal Brain Damage \citep{LeCunDenkerSolla}, the
architecture itself is unchanged.  This imposes a stringent lower
bound on the complexity of the model. Since not all machine learning
problems admit the same level of difficulty and different tasks
naturally require varying levels of complexity, complex models trained
with insufficient data can be prone to overfitting. This places a
burden on a practitioner to specify an architecture at the right level
of complexity which is often hard and requires significant levels of
experience and domain knowledge.  For this reason, network
architecture is often treated as a hyperparameter which is tuned using
a validation set. The search space can quickly become exorbitantly
large \citep{SzegedyEtal2015,HeZhangRenSun2015} and large-scale
hyperparameter tuning to find an effective network architecture is
wasteful of data, time, and resources (e.g. grid search, random search
\citep{BergstraBardenetBengioKegl2011}).

In this paper, we attempt to remedy some of these issues.  In
particular, we provide a theoretical analysis of a supervised learning
scenario in which the network architecture and parameters are learned
simultaneously. To the best of our knowledge, our results are the
first generalization bounds for the problem of \emph{structural
  learning of neural networks}.  These general guarantees can guide
the design of a variety of different algorithms for learning in this
setting.  We describe in depth two such algorithms that directly
benefit from the theory that we develop.

In contrast to enforcing a pre-specified architecture and a
corresponding fixed complexity, our algorithms learn the requisite
model complexity for a machine learning problem in an \emph{adaptive}
fashion. Starting from a simple linear model, we add more units and
additional layers as needed. \ignore{From the cognitive perspective,
  we will adapt the neural complexity and architecture to the
  difficulty of the problem. } The additional units that we add are
carefully selected and penalized according to rigorous estimates from
the theory of statistical learning.\ignore{This will serve as a
  catalyst for the sparsity of our model as well as the strong
  generalization bonds that we will be able to derive. The first
  algorithm that we present learns traditional feedforward
  architectures. The second algorithm that we develop steps outside of
  this class of models and can learn more exotic structures.}
Remarkably, optimization problems for both of our algorithms turn out
to be strongly convex and hence are guaranteed to have a unique global
solution which is in stark contrast with other methodologies for
training neural networks.

The paper is organized as follows.  In
Appendix~\ref{sec:related_work}, we give a detailed discussion of
previous work related to this topic.
Section~\ref{sec:structural_learning} describes the broad network
architecture and therefore hypothesis set that we
consider. Section~\ref{sec:scenario} provides a formal description of
our learning scenario.  In Section~\ref{sec:bounds}, we prove strong
generalization guarantees for learning in this setting which guide the
design of the algorithm described in Section~\ref{sec:algorithm} as
well as a variant described in Appendix~\ref{sec:alternatives}.  We
conclude with experimental results in Section~\ref{sec:experiments}.

\ignore{
Accepting
the general structure of a neural network as an effective parameterized
model for supervised learning, we provide a theoretical analysis of this 
model and proceed to derive an algorithm benefitting from that theory.
In the process, we introduce a framework for training neural networks that:
\begin{enumerate}
    \item uses a stable and robust algorithm with a unique solution.  
    \item can produce much sparser and/or shallower networks compared to existing methods.
    \item adapts the structure and complexity of the network to the difficulty of the particular problem at hand, with no pre-defined architecture.
    \item is accompanied and in fact motivated by strong data-dependent generalization
      bounds, validating their adaptivity and statistical efficacy.
    \item is intuitive from the cognitive standpoint that originally
      motivated neural network architectures.
\end{enumerate}}

\section{Network architecture}
\label{sec:structural_learning}

In this section, we describe the general network architecture we
consider for feedforward neural networks, thereby also defining 
our hypothesis set. To simplify the presentation, we restrict our attention to the case of
binary classification. However, all our results can be
straightforwardly extended to multi-class classification, including
the network architecture by augmenting the number of output units, and
our generalization guarantees by using existing multi-class
counterparts of the binary classification ensemble margin bounds we
use.

A common model for feedforward neural networks is the multi-layer
architecture where units in each layer are only connected to those
in the layer below.  We will consider more general architectures where
a unit can be connected to units in any of the layers below, as
illustrated by Figure~\ref{fig:arch}. In particular, the output unit
in our network architectures can be connected to any other unit. These
more general architectures include as special cases standard
multi-layer networks (by zeroing out appropriate connections) as well
as somewhat more exotic ones
\citep{HeZhangRenSun2015,HuangLiuWeinberger2016}.

\begin{figure}[t] 
\centering
\includegraphics[scale=.35]{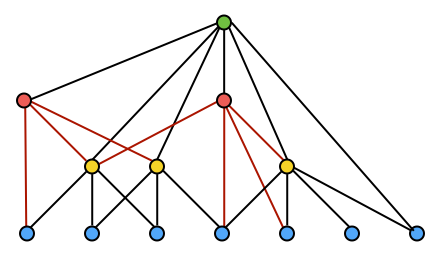}
\caption{An example of a general network architecture: output layer (green) is
connected to all of the hidden units as well as some input units. Some hidden
units (red and yellow) are connected not only to the units in the layer
directly below but to units at other levels as well.}
\vskip -.15in
\label{fig:arch}
\end{figure}

More formally, the artificial neural networks we consider are defined
as follows.  Let $l$ denote the number of intermediate layers in the
network and $n_k$ the maximum number of units in layer $k \in [l]$.
Each unit $j \in [n_k]$ in layer $k$ represents a function denoted by
$h_{k, j}$ (before composition with an activation function).  Let
$\sX$ denote the input space and for any $x \in \cX$, let
$\bPsi(x) \in \Rset^{n_0}$ denote the corresponding feature vector.
Then, the family of functions defined by the first layer functions
$h_{1, j}$, $j \in [n_1]$, is the following:
\begin{equation}
\label{eq:H1}
\sH_{1} = \set[\Big]{x \mapsto \bu \cdot \bPsi(x) 
\colon \bu \in \Rset^{n_0}, \| \bu \|_p \leq \Lambda_{1, 0}},
\end{equation}
where $p \geq 1$ defines an $l_p$-norm and $\Lambda_{1, 0} \geq 0$ is
a hyperparameter on the weights connecting layer 0 and layer 1. The family of functions $h_{k, j}$, $j \in [n_k]$,
in a higher layer $k > 1$ is then defined as follows:
\begin{align}
\label{eq:Hk}
\sH_{k} & = \set[\bigg]{x \mapsto \sum_{s = 1}^{k-1} \bu_s
  \cdot (\varphi_{s}\circ \bh_s)(x)
  \colon \nonumber \\
  &\quad\quad \bu_s \in \Rset^{n_s}, \| \bu_s \|_p \leq \Lambda_{k,s}, h_{k,s} \in
  \sH_{s}},
\end{align}
 where for each unit function
$h_{k, s}$, $\bu_s$ in \eqref{eq:Hk} denotes the vector of weights for
connections from that unit to a lower layer $s < k$. The $\Lambda_{k,
  s}$s are non-negative hyperparameters and 
$\varphi_{s}\circ \bh_s$ abusively denotes a coordinate-wise
composition:
$\varphi_{s}\circ \bh_s = (\varphi_s \circ h_{s,1}, \ldots, \varphi_s
\circ h_{s,n_s})$.  The $\varphi_s$s are assumed to be $1$-Lipschitz
activation functions.  In particular, they can be chosen to be the
\emph{Rectified Linear Unit function} (ReLU function)
$x \mapsto \max \set{0, x}$, or the \emph{sigmoid function}
$x \mapsto \frac{1}{1 + e^{-x}}$. The choice of the parameter $p \geq
1$ determines the sparsity of the network and the complexity of
the hypothesis sets $\sH_k$.

For the networks we consider, the output
unit can be connected to all intermediate units, which therefore
defines a function $f$ as follows:
\begin{equation}
\label{eq:output}
  f = \sum_{k = 1}^{l} \sum_{j = 1}^{n_k} w_{k, j} h_{k, j} =
  \sum_{k = 1}^l \bw_k \cdot \bh_k,
\end{equation}
where
$\bh_k = [h_{k,1}, \ldots, h_{k, n_k}]^\top \mspace{-8mu} \in
\sH_k^{n_k}$ and $\bw_k \in \Rset^{n_k}$ is the vector of connection
weights to units of layer $k$.  Observe that for
$\bu_s = 0$ for $s < k - 1$ and $\bw_k = 0$ for $k < l$, our
architectures coincides with standard multi-layer feedforward ones.

We will denote by $\sF$ the family of functions $f$ defined by
\eqref{eq:output} with the absolute value of the weights summing to
one:
\begin{equation*}
\sF = \set[\Bigg]{\sum_{k = 1}^l \bw_k \cdot \bh_k \colon \bh_k \in
  \sH_k^{n_k}, \sum_{k = 1}^{l}  \| \bw_k \|_1 = 1}.
\end{equation*}
Let $\tl \sH_k$ denote the union of $\sH_k$ and its reflection,
$\tl \sH_k = \sH_k \cup (-\sH_k)$, and let $\sH$ denote the union of
the families $\tl \sH_k$: $\sH = \bigcup_{k = 1}^l \tl \sH_k$.
Then, $\sF$ coincides with the convex hull of $\sH$:
$\sF = \conv(\sH)$.

For any $k \in [l]$ we will also consider the family $\sH^*_k$ derived
from $\sH_k$ by setting $\Lambda_{k, s} = 0$ for $s < k - 1$, which
corresponds to units connected only to the layer below.  We similarly
define $\tl \sH^*_k = \sH^*_k \cup (-\sH^*_k)$ and
$\sH^* = \cup_{k = 1}^l \sH^*_k$, and define the $\sF^*$ as the convex
hull $\sF^* = \conv(\sH^*)$.  Note that the architecture corresponding
to the family of functions $\sF^*$ is still more general than 
standard feedforward neural network architectures since the output
unit can be connected to units in different layers.

\section{Learning problem}
\label{sec:scenario}

We consider the standard supervised learning scenario and assume that
training and test points are drawn i.i.d.\ according to some
distribution $\sD$ over $\sX \times \set{-1, +1}$ and denote by
$S = ((x_1, y_1), \ldots, (x_m, y_m))$ a training sample of size $m$
drawn according to $\sD^m$.

For a function $f$ taking values in $\Rset$, we denote by
$R(f) = \bE_{(x,y)\sim \sD}[1_{yf(x) \leq 0}]$ its generalization
error and, for any $\rho > 0$, by $\h R_{S, \rho}(f)$ its empirical
margin error on the sample $S$:
$\h R_{S, \rho}(f) = \frac{1}{m} \sum_{i = 1}^m 1_{y_i f(x_i) \leq
  \rho}$.

The learning problem consists of using the training sample $S$ to
determine a function $f$ defined by \eqref{eq:output} with small
generalization error $R(f)$.  For an accurate predictor $f$, we expect
many of the weights to be zero and the corresponding architecture
to be quite sparse, with fewer than $n_k$ units at layer $k$ and relatively
few non-zero connections. In that sense, learning an accurate
function $f$ implies also learning the underlying architecture.

In the next section, we present data-dependent learning bounds for
this problem that will help guide the design of our algorithm.

\section{Generalization bounds}
\label{sec:bounds}

Our learning bounds are expressed in terms of the Rademacher
complexities of the hypothesis sets $\sH_k$. The empirical Rademacher
complexity of a hypothesis set $\sG$ for a sample $S$ is denoted by $\h \R_S(\sG)$
and defined as follows:
\begin{equation*}
\h \R_S(\sG) = \frac{1}{m}\E_{\ssigma}\left[ \sup_{h \in \sG}
  \sum_{i = 1}^m \sigma_i h(x_i) \right],
\end{equation*}
where $\ssigma = (\sigma_1, \ldots, \sigma_m)$, with $\sigma_i$s
independent uniformly distributed random variables taking values in
$\set{-1, +1}$.  Its Rademacher complexity is defined by $\R_m(\sG) = \E_{S \sim \sD^m} [\h \R_S(\sG)]$. These
are data-dependent complexity measures that lead to finer learning
guarantees \citep{KoltchinskiiPanchenko2002,BartlettMendelson2002}.

As pointed out earlier, the family of functions $\sF$ is the convex
hull of $\sH$. Thus, generalization bounds for ensemble methods can be
used to analyze learning with $\sF$. In particular, we can leverage
the recent margin-based learning guarantees of
\citet{CortesMohriSyed2014}, which are finer than those that can be
derived via a standard Rademacher complexity analysis
\citep{KoltchinskiiPanchenko2002}, and which admit an explicit
dependency on the mixture weights $\bw_k$ defining the
ensemble function $f$. That leads to the following learning
guarantee.

\begin{theorem}[Learning bound]
\label{th:adanet}
Fix $\rho > 0$.  Then, for any $\delta > 0$, with probability at least
$1 - \delta$ over the draw of a sample $S$ of size $m$ from $\sD^m$,
the following inequality holds for all
$f = \sum_{k = 1}^l \bw_k \cdot \bh_k \in \sF$:
\begin{align*}
R(f) 
& \leq \widehat{R}_{S, \rho}(f) + \frac{4}{\rho} \sum_{k = 1}^{l} \big \| \bw
  _k \big \|_1 \R_m(\tl \sH_k) + \frac{2}{\rho} \sqrt{\frac{\log l}{m}}\\
& \quad + C(\rho, l, m, \delta),   
\end{align*}
where
$C(\rho, l, m, \delta) = \sqrt{\big\lceil \frac{4}{\rho^2}
  \log(\frac{\rho^2 m}{\log l}) \big\rceil \frac{\log l}{m} +
  \frac{\log(\frac{2}{\delta})}{2m}} = \tl O\Big(\frac{1}{\rho}
\sqrt{\frac{\log l}{m}}\Big)$.
\end{theorem}
The proof of this result, as well as that of all other main theorems
are given in Appendix~\ref{app:theory}.  The bound of the theorem can
be generalized to hold uniformly for all $\rho \in (0, 1]$, at the
price of an additional term of the form
$\sqrt{\log \log_2 (2/\rho)/m}$ using standard
techniques \citep{KoltchinskiiPanchenko2002}.

Observe that the bound of the theorem depends only logarithmically on
the depth of the network $l$. But, perhaps more remarkably, the
complexity term of the bound is a $\| \bw_k \|_1$-weighted average of
the complexities of the layer hypothesis sets $\sH_k$, where the
weights are precisely those defining the network, or the function $f$.
This suggests that a function $f$ with a small empirical margin error
and a deep architecture benefits nevertheless from a strong
generalization guarantee, if it allocates more weights to lower layer
units and less to higher ones. Of course, when the weights are sparse,
that will imply an architecture with relatively fewer units or
connections at higher layers than at lower ones. The bound of the theorem
further gives a quantitative guide for apportioning the weights
depending on the Rademacher complexities of the layer hypothesis sets.

This data-dependent learning guarantee will serve as a foundation for
the design of our structural learning algorithms in
Section~\ref{sec:algorithm} and Appendix~\ref{sec:alternatives}.
However, to fully exploit it, the Rademacher complexity measures need
to be made explicit. One advantage of these data-dependent measures is
that they can be estimated from data, which can lead to more
informative bounds. Alternatively, we can derive useful upper bounds
for these measures which can be more conveniently used in our
algorithms. The next results in this section provide precisely such
upper bounds, thereby leading to a more explicit generalization bound.

We will denote by $q$ the conjugate of $p$, that is
$\frac{1}{p} + \frac{1}{q} = 1$, and define
$r_\infty = \max_{i \in[1, m]} \| \bPsi(x_i)\|_\infty$.

Our first result gives an upper bound on the Rademacher
complexity of $\sH_k$ in terms of the Rademacher
complexity of other layer families.
\begin{lemma}
\label{lemma:Rad_Hk_Hk-1}
For any $k > 1$, the empirical Rademacher complexity of $\sH_k$ for a
sample $S$ of size $m$ can be upper-bounded as follows in terms of
those of $\sH_s$s with $s < k$:
\begin{equation*}
\h \R_S(\sH_k) 
\leq 2 \sum_{s=1}^{k-1} \Lambda_{k,s}
n_{s}^{\frac{1}{q}} \h \R_S(\sH_{s}).
\end{equation*}
\end{lemma}
For the family $\sH^*_k$, which is directly relevant to many of our
experiments, the following more explicit upper bound can be derived,
using Lemma~\ref{lemma:Rad_Hk_Hk-1}.

\begin{lemma}
\label{lemma:Rad_Hk}
Let $\Lambda_k \!= \prod_{s = 1}^k 2 \Lambda_{s, s - 1}$ and
$N_k \!= \prod_{s = 1}^k n_{s - 1}$. Then, for any $k \geq 1$, the
empirical Rademacher complexity of $\sH_k^*$ for a sample $S$ of size
$m$ can be upper bounded as follows:
\begin{equation*}
\h \R_S(\sH_k^*) 
\leq r_\infty \Lambda_k N_k^{\frac{1}{q}} 
\sqrt{\frac{\log (2 n_0)}{2 m}}.
\end{equation*}
\end{lemma}
Note that $N_k$, which is the product of the number of units in layers
below $k$, can be large. This suggests that values of $p$ closer to
one, that is larger values of $q$, could be more helpful to control
complexity in such cases.  More generally, similar explicit upper
bounds can be given for the Rademacher complexities of subfamilies of
$\sH_k$ with units connected only to layers $k, k - 1, \ldots, k - d$,
with $d$ fixed, $d < k$.  Combining Lemma~\ref{lemma:Rad_Hk} with
Theorem~\ref{th:adanet} helps derive the following explicit learning
guarantee for feedforward neural networks with an output unit
connected to all the other units.

\begin{corollary}[Explicit learning bound]
\label{cor:adanet}
Fix $\rho > 0$.  Let
$\Lambda_k = \prod_{s = 1}^k 4 \Lambda_{s, s - 1}$ and
$N_k \!= \prod_{s = 1}^k n_{s - 1}$. Then, for any $\delta > 0$, with
probability at least $1 - \delta$ over the draw of a sample $S$ of
size $m$ from $\sD^m$, the following inequality holds for all
$f = \sum_{k = 1}^l \bw_k \cdot \bh_k \in \sF^*$:
\begin{align*}
R(f) 
& \leq \widehat{R}_{S, \rho}(f) + \frac{2}{\rho} \sum_{k = 1}^{l} \big \| \bw
  _k \big \|_1 \bigg[\ov r_\infty \Lambda_k N_k^{\frac{1}{q}} \sqrt{\frac{2 \log (2 n_0)}{m}}\bigg] \\
& \quad + \frac{2}{\rho} \sqrt{\frac{\log l}{m}} + C(\rho, l, m, \delta),   
\end{align*}
where $C(\rho, l, m, \delta) = \sqrt{\big\lceil \frac{4}{\rho^2}
  \log(\frac{\rho^2 m}{\log l}) \big\rceil \frac{\log l}{m} +
  \frac{\log(\frac{2}{\delta})}{2m}} = \tl O\Big(\frac{1}{\rho}
\sqrt{\frac{\log l}{m}}\Big)$, and where $\ov r_\infty = \E_{S \sim \sD^m}[r_\infty]$.
\end{corollary}
The learning bound of Corollary~\ref{cor:adanet} is 
a finer guarantee than previous ones by \citep{Bartlett1998},
\cite{NeyshaburTomiokaSrebro2015}, or \cite{SunChenWangLiu2016}. 
This is because it explicitly 
differentiates between the weights of different
layers while previous bounds treat all weights indiscriminately.
This is crucial to the design of algorithmic design since the network
complexity no longer needs to grow exponentially as a function of
depth.  Our bounds are also more general and apply to more
other network architectures, such as those introduced in
\citep{HeZhangRenSun2015,HuangLiuWeinberger2016}.

\ignore{The generalization bound above informs us that the complexity of the
neural network returned is a weighted combination of the complexities
of each unit in the neural network, where the weights 
are precisely the ones that define our network. Specifically, this again agrees
with our intuition that deeper networks are more complex and suggests
that if we can find a model that has both small empirical error and
most of its weight on the shallower units, then such a model will
generalize well.

Toward this goal, we will design an algorithm that directly seeks to
minimize upper bounds of this generalization bound. In the process,
our algorithm will train neural networks that discriminate deeper
networks from shallower ones. This is a novel property that existing
regularization techniques in the deep learning toolbox do not enforce.
Techniques such as $l_2$ and $l_1$ regularization and
dropout (see e.g. \cite{GoodfellowBengioCourville2016}) are generally
applied uniformly across all units in the network.}
 
\section{Algorithm}
\label{sec:algorithm}

This section describes our algorithm, \textsc{AdaNet}, for
\emph{adaptive} learning of neural networks. \textsc{AdaNet}
adaptively grows the structure of a neural network, balancing model
complexity with empirical risk minimization.  We also describe in
detail in Appendix~\ref{sec:alternatives} another variant of
\textsc{AdaNet} which admits some favorable properties.

Let $x \mapsto \Phi(-x)$ be a non-increasing convex function
upper-bounding the zero-one loss, $x \mapsto 1_{x \leq 0}$, such that
$\Phi$ is differentiable over $\Rset$ and $\Phi'(x) \neq 0$ for all
$x$. This surrogate loss $\Phi$ may be, for instance, the exponential
function $\Phi(x) = e^x$ as in AdaBoost 
\citet{FreundSchapire97}, or the logistic function,
$\Phi(x) = \log(1 + e^x)$ as in logistic regression.

\subsection{Objective function}

Let $\set{h_1, \ldots, h_N}$ be a subset of $\sH^*$. In the most
general case, $N$ is infinite. However, as discussed later, in
practice, the search is limited to a finite set.  For any $j \in [N]$, we will denote by $r_j$ the Rademacher complexity of the family
$\sH_{k_j}$ that contains $h_j$: $r_j = \R_m(\sH_{k_j})$.

\textsc{AdaNet} seeks to find a function 
$f = \sum_{j = 1}^N w_j h_j \in \sF^*$ (or neural network) that directly minimizes the
data-dependent generalization bound of Corollary~\ref{cor:adanet}.
This leads to the following objective function:
\begin{align}
\label{eq:objective}
F(\bw) = \frac{1}{m} \sum_{i = 1}^m \Phi\Big(1
    - y_i \sum_{j = 1}^N w_j h_j \Big) + \sum_{j = 1}^N \Gamma_j |w_j|, 
\end{align}
where $\bw \in \Rset^{N}$ and $\Gamma_j = \lambda r_j + \beta$, with
$\lambda \geq 0$ and $\beta \geq 0$ hyperparameters. 
The objective function \eqref{eq:objective} is a convex function of
$\bw$. It is the sum of a convex surrogate of the empirical error and
a regularization term, which is a weighted-$l_1$ penalty containing
two sub-terms: a standard norm-1 regularization which admits $\beta$
as a hyperparameter, and a term that discriminates the functions $h_j$
based on their complexity.

The optimization problem consisting of minimizing the objective
function $F$ in \eqref{eq:objective} is defined over a very large space
of base functions $h_j$.  \textsc{AdaNet} consists of applying
coordinate descent to \eqref{eq:objective}. In that sense, our
algorithm is similar to the DeepBoost algorithm of
\citet{CortesMohriSyed2014}. However, unlike DeepBoost, which combines
decision trees, \textsc{AdaNet} learns a deep neural network, which
requires new methods for constructing and searching the space of
functions $h_j$. Both of these aspects differ significantly from the
decision tree framework. In particular, the search is particularly
challenging.  In fact, the main difference between the algorithm
presented in this section and the variant described in
Appendix~\ref{sec:alternatives} is the way new candidates $h_j$ are
examined at each iteration.

\begin{figure}[t] 
\centering
\begin{tabular}{c@{\hspace{1cm}}c}
  (a) & \includegraphics[scale=.2]{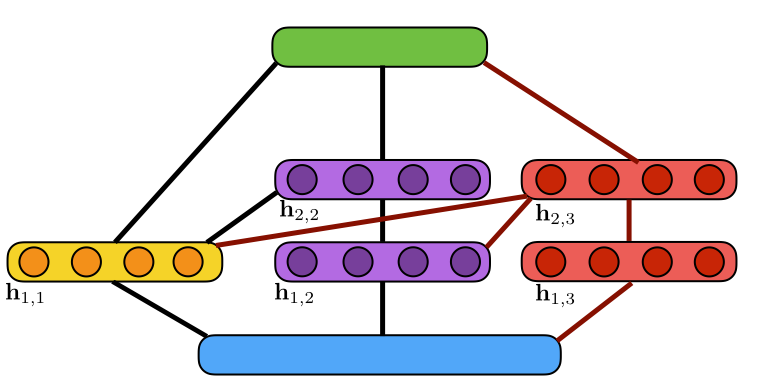} \\
  (b) & \includegraphics[scale=.2]{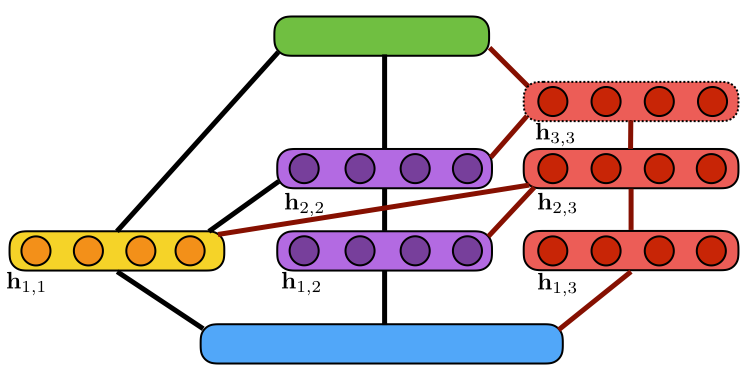} \\
\end{tabular}
\caption{Illustration of the algorithm's incremental construction of a
  neural network.  The input layer is indicated in blue, the output
  layer in green. Units in the yellow block are added at the first
  iteration while units in purple are added at the second iteration.
  Two candidate extensions of the architecture are considered at the
  the third iteration (shown in red): (a) a two-layer extension; (b) a
  three-layer extension. Here, a line between two blocks of units
  indicates that these blocks are fully-connected.}
\vskip -.25in
\label{fig:adanet}
\end{figure}

\subsection{Description}

We start with an informal description of \textsc{AdaNet}.  Let
$B \geq 1$ be a fixed parameter determining the number of units per
layer of a candidate subnetwork.  The algorithm proceeds in $T$
iterations.  Let $l_{t - 1}$ denote the depth of the neural network
constructed before the start of the $t$-th iteration.  At iteration
$t$, the algorithm selects one of the following two options:

1. augmenting the current neural network with a subnetwork with the
same depth as that of the current network
$\bh \in \sH^{* B}_{l_{t - 1}}$, with $B$ units per layer.  Each unit
in layer $k$ of this subnetwork may have connections to existing units
in layer $k-1$ of \textsc{AdaNet} in addition to connections to units
in layer $k-1$ of the subnetwork.

2.  augmenting the current neural network with a deeper subnetwork
(depth $l_{t - 1} + 1$) $\bh' \in \sH^{* B}_{l_{t - 1}}$. The set of
allowed connections is defined the same way as for $\bh$.

The option selected is the one leading to the best reduction of the
current value of the objective function, which depends both on the
empirical error and the complexity of the subnetwork added, which is
penalized differently in these two options.

Figure~\ref{fig:adanet} illustrates this construction and the two
options just described.  An important aspect of our algorithm is that
the units of a subnetwork learned at a previous iteration (say
$\bh_{1,1}$ in Figure~\ref{fig:adanet}) can serve as input to deeper
subnetwork added later (for example $\bh_{2, 2}$ or $\bh_{2, 3}$ in
the Figure).  Thus, the deeper subnetworks added later can take
advantage of the embeddings that were learned at the previous iterations.
The algorithm terminates after $T$ rounds or if the \textsc{AdaNet}
architecture can no longer be extended to improve the objective
\eqref{eq:objective}.

\ignore{
rounds $t=1, \ldots, T$.  At the first round
$t=1$, we consider two models: a linear model and a neural network
with one hidden layer. This neural network can be viewed as a vector
of elements of $\sH_2^{(p)}$:
$\bh_{1,1} = (h_{1,1,1}, \ldots, h_{1,B,1})$ where $B$ is a size of
the hidden layer and $h_{k, j, t}$ denotes $j$-th unit in the
$k$-layer introduced to the model on $t$-th iteration.  Similarly, the
linear model is an element of $\sH_1^{(p)}$. The model that provides
the best improvement in \eqref{eq:objective}, that is the model with
the best tradeoff between empirical error and complexity, becomes a
part of the \textsc{AdaNet} architecture.  Now suppose that at round
$t=1$, we have selected a model with one hidden layer.  Then at the
next round we consider two subnetworks, with one and two hidden layers
respectively.  An important aspect of our algorithm is that the
$\bh_{1,1} = (h_{1,1,1}, \ldots, h_{1,B,1})$ that was learned on the
first iteration serves as an input to the last hidden layer of the
subnetwork with two hidden layers. In this way, the new units that we
introduce to the model can take advantage of the embeddings that we
have learned at the previous iterations. Similarly, if, for $t=2$, the
subnetwork of depth two provides a greater improvement in
\eqref{eq:objective}, then, at the next iteration, we choose between
subnetworks of depth two and three. In this case,
$(\bh_{1,1}, \bh_{1,2}) = (h_{1,1,1}, \ldots, h_{1,B,1}, h_{1, 1, 2},
\ldots, h_{1, B, 2})$ serves as an input to the second layer of each
of these subnetworks.  The second layer
$\bh_{2, 2} = (h_{2,1,2}, \ldots, h_{2,B,2})$ of the subnetwork added
at $t=2$ serves as an input to the third layer of subnetwork of depth
three.  The algorithm terminates after $T$ rounds or if we can no
longer extend the \textsc{AdaNet} architecture to improve
\eqref{eq:objective}.  See Figure~\ref{fig:adanet} for an example.
}

More formally, \textsc{AdaNet} is a boosting-style algorithm that
applies (block) coordinate descent to \eqref{eq:objective}.  At each
iteration of block coordinate descent, descent coordinates $\bh$ (base
learners in the boosting literature) are selected from the space of
functions $\sH^*$.  These coordinates correspond to the direction of
the largest decrease in \eqref{eq:objective}.  Once these coordinates
are determined, an optimal step size in each of these directions is
chosen, which is accomplished by solving an appropriate convex
optimization problem.

Note that, in general, the search for the optimal descent coordinate
in an infinite-dimensional space or even in finite but large sets such
as that of all decision trees of some large depth may be intractable,
and it is common to resort to a heuristic search (weak learning
algorithm) that returns $\delta$-optimal coordinates. For instance, in
the case of boosting with trees one often grows trees according to
some particular heuristic \cite{FreundSchapire97}.

We denote the \textsc{AdaNet} model after $t - 1$
rounds by $f_{t - 1}$, which is parameterized by $\bw_{t - 1}$.
Let $\bh_{k, t-1}$ denote the vector of outputs of units in the $k$-th
layer of the \textsc{AdaNet} model, $l_{t-1}$ be the depth of the 
\textsc{AdaNet} architecture, $n_{k, t-1}$ 
be the number of units in $k$-th layer after $t-1$ rounds.
At round $t$, we select descent coordinates
by considering two candidate subnetworks
$\bh \in \widetilde{\sH}_{l_{t-1}}^*$ and
$\bh' \in  \widetilde{\sH}_{l_{t-1} + 1}^*$
that are generated by a weak learning algorithm
\textsc{WeakLearner}. Some choices for this algorithm
in our setting are described below.\ignore{,
where $\widetilde{\sH}^{(p)}_{k}$ is defined to be:
\begin{align}
\label{eq:Hk_tilde}
& \set[\bigg]{x \mapsto  \bu' \cdot (\varphi_{k-1}\circ \bh')(x) + 
 \bu \cdot (\varphi_{k-1}\circ \bh_{k-1, t-1})(x) \colon \nonumber \\
&\quad\quad\quad   (\bu', \bu) \in \Rset^{B + n_{k-1, t-1}}, h'_j \in \widetilde{\sH}^{(p)}_{k-1}},
\end{align}
In other words, $\widetilde{\sH}^{(p)}_{k}$ can viewed as a
restriction of $\sH_{k}$ to functions that only use inputs from the
layer directly below, and this layer, itself, consists of the units
that have been previously added to \textsc{AdaNet} architecture as
well as some newly introduced ones. There are multiple ways of
generating candidate subnetworks, and we discuss some of them below.}
Once we obtain $\bh$ and $\bh'$, we select one of these vectors of
units, as well as a vector of weights $\bw \in \Rset^B$, so that the
result yields the best improvement in \eqref{eq:objective}. This is
equivalent to minimizing the following objective function over
$\bw \in \Rset^B$ and $\bu \in \{\bh, \bh'\}$:
\begin{align}
\label{eq:objective_t}
F_t(\bw, \bu) 
 \!=\! \frac{1}{m} \sum_{i = 1}^m \Phi\Big( 1 - y_i f_{t - 1}(x_i)
&- y_i \bw \cdot \bu(x_i) \Big) 
\nonumber  \\ & \quad+ \Gamma_{\bu} \| \bw \|_1,
\end{align}
where $\Gamma_\bu  = \lambda r_\bu  + \beta$ and $r_\bu$
is $\R_m\big(\sH_{l_{t-1}}\big)$ if $\bu = \bh$ and
$\R_m\big(\sH_{l_{t-1} + 1}\big)$ otherwise.
In other words, if $\min_{\bw} F_t(\bw, \bh) \leq \min_{\bw} F_t(\bw, \bh')$,
then
\begin{align*}
  \bw^* = \argmin_{\bw \in \Rset^B} {F_t(\bw, \bh)}, \quad   \bh_t = \bh
\end{align*}
and otherwise
\begin{align*}
  &\bw^* = \argmin_{\bw \in \Rset^B} {F_t(\bw, \bh')}, \quad \bh_t = \bh'
\end{align*}
If $F(\bw_{t-1} + \bw^*) < F(\bw_{t-1})$ then we set
$f_{t-1} = f_t + \bw^* \cdot \bh_t$ and otherwise we terminate
the algorithm.

\begin{figure}[t]
  \begin{ALGO}{AdaNet}{S = ((x_i, y_i)_{i = 1}^m}
  \SET{f_0}{0}

\DOFOR{t \EQ 1 \TO T}
  \SET{\bh, \bh'}{\textsc{WeakLearner}\big(S, f_{t-1}\big)}
  \SET{\bw}{\textsc{Minimize}\big(F_t(\bw, \bh)\big)}
  \SET{\bw'}{\textsc{Minimize}\big(F_t(\bw, \bh')\big)}
  \IF{F_t(\bw, \bh') \leq F_t(\bw', \bh')}
  \SET{\bh_t}{\bh}
  \ELSE
  \SET{\bh_t}{\bh'}
  \FI
  \IF{F(\bw_{t-1} + \bw^*) < F(\bw_{t-1})}
  \SET{f_{t-1}}{f_t + \bw^* \cdot \bh_t}
  \ELSE
  \RETURN{f_{t-1}}
  \FI   
\OD
\RETURN{f_T}
\end{ALGO}
\vskip -.5cm
\caption{Pseudocode of the AdaNet algorithm. On line 3 two candidate
 subnetworks are generated (e.g. randomly or by solving
 \eqref{eq:subnetwork_objective}). On lines 3 and 4, \eqref{eq:objective_t}
 is solved for each of these candidates. On lines 5-7 the best
 subnetwork is selected and on lines 9-11 termination condition
 is checked.}
\label{algo:adanet}
\end{figure}

\begin{table*}[t]
  \caption{Experimental results for \textsc{AdaNet}, \textsc{NN},
    \textsc{LR} and \textsc{NN-GP} for different
    pairs of labels in {\tt CIFAR-10}. Boldfaced results are
    statistically significant at a 5\% confidence level.}
  \label{tbl:adanet_vs_dnn}
  \vskip 0.15in
  \begin{center}
    \begin{small}
      \begin{tabular}{lcccc}
        \hline
        \abovespace\belowspace
        Label pair & \textsc{AdaNet} & \textsc{LR} & \textsc{NN} & \textsc{NN-GP} \\
        \hline
        \abovespace
        {\tt deer-truck}        & {\bf 0.9372 $\pm$ 0.0082} & 0.8997 $\pm$ 0.0066 & 0.9213 $\pm$ 0.0065 & 0.9220 $\pm$ 0.0069 \\
        {\tt deer-horse}        & {\bf 0.8430 $\pm$ 0.0076} & 0.7685 $\pm$ 0.0119 & 0.8055 $\pm$ 0.0178 & 0.8060 $\pm$ 0.0181 \\
        {\tt automobile-truck}  & {\bf 0.8461 $\pm$ 0.0069} & 0.7976 $\pm$ 0.0076 & 0.8063 $\pm$ 0.0064 & 0.8056 $\pm$ 0.0138 \\
        {\tt cat-dog}           & {\bf 0.6924 $\pm$ 0.0129} & 0.6664 $\pm$ 0.0099 & 0.6595 $\pm$ 0.0141 & 0.6607 $\pm$ 0.0097 \\
        {\tt dog-horse}         & {\bf 0.8350 $\pm$ 0.0089} & 0.7968 $\pm$ 0.0128 & 0.8066 $\pm$ 0.0087 & 0.8087 $\pm$ 0.0109 \\
        \hline
      \end{tabular}
    \end{small}
  \end{center}
  \vskip -0.2in
\end{table*}

There are many different choices for the \textsc{WeakLearner}
algorithm.  For instance, one may generate a large number of random
networks and select the one that optimizes
\eqref{eq:objective_t}. Another option is to directly minimize
\eqref{eq:objective_t} or its regularized version:
\begin{align}
\label{eq:subnetwork_objective}
\widetilde{F}_t(\bw, \bh) 
 \!=\! \frac{1}{m} \sum_{i = 1}^m \Phi\Big( 1 - y_i f_{t - 1}(x_i)
&- y_i \bw \cdot \bh(x_i) \Big) 
  \nonumber\\ &\quad+ \mathcal{R}(\bw, \bh),
\end{align}
over both $\bw$ and $\bh$.
Here $\mathcal{R}(\bw ,\bh)$ is a regularization term that,
for instance, can be used to enforce
that $\|\bu_s\|_p \leq \Lambda_{k,s}$ in \eqref{eq:Hk}.
Note that, in general, \eqref{eq:subnetwork_objective} is
a non-convex objective. However, we do not rely
on finding a global solution to the corresponding optimization
problem. In fact, standard guarantees for regularized
boosting only require that each $\bh$ that is added to the model
decreases the objective by a constant amount (i.e. it satisfies
$\delta$-optimality condition) for a boosting algorithm
to converge \citep{RatschMikaWarmuth2001,LuoTseng1992}.

Furthermore, the algorithm that we present in
Appendix~\ref{sec:alternatives} uses a weak-learning algorithm that
solves a convex sub-problem at each step and that additionally has a
closed-form solution. This comes at the cost of a more restricted
search space for finding a descent coordinate at each step of the
algorithm.

We conclude this section by observing that in our description of
\textsc{AdaNet} we have fixed $B$ for all iterations
and only two candidate subnetworks are considered at each step.
Our approach easily extends to an arbitrary number of candidate
subnetworks (for instance of different depth $l$) as well
as varying number of units per layer $B$. Furthermore,
selecting an optimal subnetwork among the candidates is
easily parallelizable allowing for efficient and effective search
for optimal descent directions.

\section{Experiments}
\label{sec:experiments}

In this section we present the results of our experiments with
\textsc{AdaNet} algorithm.

\ignore{
To show the performance of our
proposed system on CIFAR-10~\cite{Krizhevsky09learningmultiple} on two
set of experiments. Firstly, we present the benchmark of multiple
binary classification tasks and secondly, we present an exhaustive
system exploration where different configurations have competed with
each other.}

\subsection{CIFAR-10}

In our first set of experiments, we used the {\tt CIFAR-10} dataset
\cite{Krizhevsky09learningmultiple}.  This dataset consists of
$60\mathord,000$ images evenly categorized in $10$ different classes.
To reduce the problem to binary classification we considered five
pairs of classes: {\tt deer-truck}, {\tt deer-horse}, {\tt
  automobile-truck}, {\tt cat-dog}, {\tt dog-horse}.  Raw images have
been preprocessed to obtain color histograms and histogram of gradient
features.  The result is $154$ real valued features with ranges
$[0,1]$.

We compared \textsc{AdaNet} to standard
feedforward neural networks (\textsc{NN}) and logistic regression
(\textsc{LR}) models.
Note that convolutional neural networks are often
a more natural choice for image classification problems
such as {\tt CIFAR-10}. However, the goal of these
experiments is not to obtain state-of-the-art results
for this particular task, but to provide a proof-of-concept
illustrating that our structural learning approach can be
competitive with traditional approaches
for finding efficient architectures and training corresponding
networks.

Note that \textsc{AdaNet} algorithm requires the knowledge
of complexities $r_j$, which in certain cases can be estimated from data.
In our experiments, we have used the upper bound in Lemma~\ref{lemma:Rad_Hk}.
Our algorithm admits a number of hyperparameters:
regularization hyperparameters $\lambda$, $\beta$,
number of units $B$ in each layer of new subnetworks that are used
to extend the model at each iteration and
a bound $\Lambda_{k}$ on weights $(\bu',\bu)$  in each unit.
As discussed in Section~\ref{sec:algorithm}, there are different approaches
to finding candidate subnetworks in each iteration. In our experiments,
we searched for candidate subnetworks by minimizing
\eqref{eq:subnetwork_objective} with $\mathcal{R} = 0$.
This also requires a learning rate hyperparameter $\eta$.
These hyperparamers have been optimized over the following ranges:
$\lambda \in \{0, 10^{-8}, 10^{-7}, 10^{-6}, 10^{-5}, 10^{-4}\}$,
$B \in \{100, 150, 250\}$, $\eta \in \{10^{-4}, 10^{-3}, 10^{-2}, 10^{-1}\}$.
We have used a single $\Lambda_k$ for all $k > 1$
optimized over $\{1.0, 1.005, 1.01, 1.1, 1.2\}$. For simplicity,
$\beta = 0$.

Neural network models also admit learning rate $\eta$ and regularization
coefficient $\lambda$ as hyperparameters, as well as the number of
hidden layers $l$ and number of units $n$ in each hidden layer.
The range of $\eta$ was the same as for \textsc{AdaNet}
and we varied $l$ in $\{1,2,3\}$, $n$ in $\{100, 150, 512, 1024, 2048\}$
and $\lambda \in \{0, 10^{-5}, 10^{-4}, 10^{-3}, 10^{-2}, 10^{-1}\}$.
Logistic regression only admits $\eta$ and $\lambda$ as its
hyperparameters that were optimized over the same ranges.
Note that the total number of hyperparameter settings
for \textsc{AdaNet} and standard neural networks is exactly
the same. Furthermore, the same holds for the number of hyperparameters
that determine resulting architecture of the model:
$\Lambda$ and $B$ for \textsc{AdaNet} and $l$ and $n$ for neural
network models. Observe that while a particular setting of $l$ and $n$
determines a fixed architecture $\Lambda$ and $B$ parameterize
a structural learning procedure that may result in a different
architecture depending on the data.

In addition to the grid search procedure,
we have conducted a hyperparameter optimization
for neural networks using Gaussian process
bandits (\textsc{NN-GP}),
which is a sophisticated Bayesian non-parametric method for
response-surface modeling in conjunction with a bandit
algorithm~\cite{Snoek12bandit}. Instead of operating
on a pre-specified grid, this allows one to search for hyperparameters
in a given range. We used the following ranges:
$\lambda \in [10^{-5}, 1]$, $\eta \in [10^{-5}, 1]$,
$l \in [1,3]$ and $n \in [100,2048]$. This algorithm was run
for 500 trials which is more than the number of hyperparameter settings
considered by \textsc{AdaNet} and \textsc{NN}. 
Observe that this search procedure can also
be applied to our algorithm but we choose not to do it
in this set of experiments to further demonstrate competitiveness
of the structural learning approach.

In all experiments we use ReLu activations. \textsc{NN},
\textsc{NN-GP} and \textsc{LR} are trained using
stochastic gradient method with batch size of 100 and maximum of
10,000 iterations. The same configuration is used for solving
\eqref{eq:subnetwork_objective}. We use $T=30$ for \textsc{AdaNet}
in all our experiments although in most cases algorithm
terminates after 10 rounds.

In each of the experiments, we used standard 10-fold cross-validation
for performance evaluation and model selection. In particular, the
dataset was randomly partitioned into 10 folds, and each algorithm was
run 10 times, with a different assignment of folds to the training
set, validation set and test set for each run. Specifically, for each
$i \in \{0, \ldots, 9\}$, fold $i$ was used for testing, fold $i +
1~(\text{mod}~10)$ was used for validation, and the remaining folds
were used for training. For each setting of the parameters, we
computed the average validation error across the 10 folds, and
selected the parameter setting with maximum average accuracy across
validation folds. We report average accuracy (and standard deviations)
of the selected hyperparameter setting across test folds
in Table~\ref{tbl:adanet_vs_dnn}.

\ignore{
The set of
hyperparameters studied for \textsc{AdaNet} are learning rate between
$[10^{-5}, 0.1]$, 10,000 iterations per subnetwork, 100 as batch size,
nodes per layer between [100, 250], cost per layer between [1.0, 1.2]
with 0.05 increments, complexity regularization in $[0, 10^{-8},
  10^{-7}, 10^{-6}, 10^{-5}, 10^{-4}]$ and 3,000 iterations for the
output layer. Neural networks and logistic regression have been
optimized with the following hyperparameters range: batch size in [5,
  200], learning rate between $[10^{-5}, 2]$, L1 and L2
regularizations between $[10^{-4}, 10]$. In addition, neural networks
are also optimized using between 1 to 3 hidden layers, number of nodes
in each hidden layer ranged from 20 to 2048.} 

Our results show that \textsc{AdaNet} outperforms other methods
on each of the datasets.
The average architectures for all label
pairs are provided in Table~\ref{tbl:adanet_archs}.
Note that \textsc{NN} and \textsc{NN-GP} always selects
one layer architecture. The architectures selected by \textsc{AdaNet}
typically also have just one layer and fewer nodes than those
selected by \textsc{NN} and \textsc{NN-GP}. However,
on a more challenging problem {\tt cat-dog} \textsc{AdaNet}
opts for a more complex model with two layers which results in
a better performance. This further illustrates that
our approach allows to learn network architectures in
adaptive fashion depending on the complexity of the given problem.

\begin{table}[t]
  \caption{Average number of units in each layer.}
  \label{tbl:adanet_archs}
  \vskip 0.05in
  \begin{center}
    \begin{small}
      \begin{tabular}{l@{\hspace{.075cm}}c@{\hspace{.075cm}}c@{\hspace{.075cm}}c@{\hspace{.075cm}}c@{\hspace{.075cm}}}
        \hline
        \abovespace\belowspace
        Label pair & \multicolumn{2}{c}{\textsc{AdaNet}}  & \textsc{NN} & \textsc{NN-GP} \\
        \hline
          & 1st layer & 2nd layer & & \\
        \hline
        \abovespace
        {\tt deer-truck}        & 990 & 0 & 2048 & 1050  \\
        {\tt deer-horse}        & 1475 & 0 & 2048 & 488 \\
        {\tt automobile-truck}  & 2000 & 0 & 2048 & 1595 \\
        {\tt cat-dog}           & 1800 & 25 & 512 & 155 \\
        {\tt dog-horse}         & 1600 & 0 & 2048 & 1273 \\
        \hline
      \end{tabular}
    \end{small}
  \end{center}
  \vskip -0.15in
\end{table}

\begin{table}[t]
  \caption{Experimental results for different variants of\textsc{AdaNet}.}
  \label{tbl:adanet_systems}
  \vskip 0.15in
  \begin{center}
    \begin{small}
      \begin{tabular}{cc}
        \hline
        \abovespace\belowspace
        Algorithm & Accuracy ($\pm$ std. dev.) \\
        \hline
        \abovespace
        \ignore{Baseline & 0.9354 $\pm$ 0.0082 \\}
        \ignore{R & 0.9370 $\pm$ 0.0082 \\}
        \textsc{AdaNet.SD} & 0.9309 $\pm$ 0.0069 \\
        \textsc{AdaNet.R} & 0.9336 $\pm$ 0.0075 \\
        \textsc{AdaNet.P} & 0.9321 $\pm$ 0.0065 \\
        \textsc{AdaNet.D} & 0.9376 $\pm$ 0.0080 \\
        \hline
      \end{tabular}
    \end{small}
  \end{center}
  \vskip -0.2in
\end{table}

As discussed in
Section~\ref{sec:algorithm}, various different heuristics
can be used to generate candidate subnetworks on each
iteration of \textsc{AdaNet}.
In the second set of experiments we have varied objective
function \eqref{eq:subnetwork_objective}, as well as
the domain over which it is optimized. This allows
us to study sensitivity of \textsc{AdaNet}
to the choice of heuristic that is used to generate
candidate subnetworks.
In particular, we have considered the following variants
of \textsc{AdaNet}.
\textsc{AdaNet.R} uses $\mathcal{R}(\bw, \bh) = \Gamma_\bh \|w\|_1$
as a regularization term in \eqref{eq:subnetwork_objective}.
As \textsc{AdaNet} architecture grows, each new subnetwork
is connected to all the previous subnetworks which
significantly increases the number of connections
in the network and overall complexity of the model.
\textsc{AdaNet.P} and \textsc{AdaNet.D} are
restricting connections to existing subnetworks
in different ways.
\textsc{AdaNet.P}
connects each new subnetwork only to subnetwork that was
added on the previous iteration.
\textsc{AdaNet.D} uses dropout on the connections
to previously added subnetworks.

Finally, \textsc{AdaNet}
uses an upper bound on Rademacher complexity
from Lemma~\ref{lemma:Rad_Hk}.
\textsc{AdaNet.SD} uses standard deviations of the outputs
of the last hidden layer on the training data as surrogate
for Rademacher complexities. The advantage of using
this data-dependent measure of complexity is that it
eliminates hyperparameter $\Lambda$ reducing the hyperparameter
search space.
We report average accuracies across test folds
for {\tt deer-truck} pair in Table~\ref{tbl:adanet_systems}.

\ignore{
In this case the hyperparameter boundaries used for the study involve
using a learning rate in between [0.0001, 0.1], subnetworks training
iterations between 100 and 10,000, a batch size of 100, number of
nodes per layer from 5 to 150, a cost per layer from 1.01 to 1.4 in
0.05 increments, complexity regularization goes from 0 to $10^{-4}$,
dropout probability in the range [0.9, 0] and 10,000 iteration to
train the output layer.

As it can be seen in Table~\ref{tbl:adanet_systems} performance is
similar among all systems. In terms of performance on the
cross-validation study systems R and D are preferred whereas system S
and P are compelling since they tend to produce smaller networks with
fewer number of subnetworks (e.g. in this dataset a minimum of 60\%
fewer subnetworks compared to the rest of the systems).}

\subsection{Criteo Click Rate Prediction}

We also compared \textsc{AdaNet} to \textsc{NN}
on Criteo Click Rate Prediction
dataset~{\small \url{https://www.kaggle.com/c/criteo-display-ad-challenge}}.
This dataset consists of 7 days of data where each instance
is an impression and a binary label (clicked or not clicked).
Each impression has 13 count features and 26 categorical features.
Count features have been transformed by taking the natural logarithm.
The values of categorical features appearing less than 100 times are replaced
by zeros.
The rest of the values are then converted to integers
which are then used as keys to look up embeddings
(that are trained together with each model).
If the number of possible values for a feature $x$ is $d(x)$,
then embedding dimension is set to $6d(f)^{1/4}$ for $d(f) > 25$.
Otherwise, embedding dimension is $d(f)$. Missing feature values
are set to zero.

We have split training set provided in the link above into
training, validation and test set.\footnote{Note that test
set provided in this link does not have ground truth labels
and can not be used in our experiments.} Our training
set received the first 5 days of data
(32,743,299 instances) and validation and test
sets consist of 1 day of data (6,548,659 instances).

Gaussian processes bandits were used to find the best hyperparameter
settings on validation set both for \textsc{AdaNet} and \textsc{NN}.
For \textsc{AdaNet} we have optimized over the following
hyperparameter ranges: $B \in {125, 256, 512}$,
$\Lambda \in [1, 1.5]$, $\eta \in [10^{-4}, 10^{-1}]$,
$\lambda \in [10^{-12},10^{-4}]$.  For \textsc{NN} the ranges were as
follows: $l \in [1,6]$, $n \in [250,512,1024,2048]$,
$\eta \in [10^{-5}, 10^{-1}]$, $\lambda \in [10^{-6},10^{-1}]$.  We
train \textsc{NN}s for 100,000 iterations using mini-batch stochastic
gradient method with batch size of 512. Same configuration is used at
each iteration of \textsc{AdaNet} to solve
\eqref{eq:subnetwork_objective}. The maximum number of hyperparameter
trials is 2,000 for both methods.  Results are presented in
Table~\ref{tbl:criteo}.  In this experiment, \textsc{NN} chooses
architecture with four hidden layer and 512 units in each hidden
layer. Remarkbly, \textsc{AdaNet} achieves a better
accuracy with an architecture consisting of single
layer with just 512 nodes. While the difference
in performance appears small it is statistically significant
on this challenging task.

\begin{table}[t]
  \caption{Experimental results for Criteo dataset.}
  \label{tbl:criteo}
  \vskip 0.15in
  \begin{center}
    \begin{small}
      \begin{tabular}{cc}
        \hline
        \abovespace\belowspace
        Algorithm & Accuracy \\
        \hline
        \abovespace
        \textsc{AdaNet} & 0.7846  \\
        \textsc{NN} & 0.7811\ignore{08617782593} \\
        \hline
      \end{tabular}
    \end{small}
  \end{center}
  \vskip -0.25in
\end{table}

\section{Conclusion}

We presented a new framework and algorithms for adaptively learning
artificial neural networks.  Our algorithm, \textsc{AdaNet}, benefits
from strong theoretical guarantees. It simultaneously learns a neural
network architecture and its parameters by balancing a trade-off
between model complexity and empirical risk minimization.  The
data-dependent generalization bounds that we presented can help guide
the design of alternative algorithms for this problem.  We reported
favorable experimental results demonstrating that our algorithm is
able to learn network architectures that perform better than the ones
found via grid search.  Our techniques are general and can be applied
to other neural network architectures such as CNNs and RNNs.

\newpage

\ignore{
\section*{Acknowledgments} 

This work was partly funded by NSF IIS-1117591 and CCF-1535987.

}

\bibliography{adanet}
{\small\bibliographystyle{icml2017}}

\newpage
\appendix

\section{Related work}
\label{sec:related_work}

There have been several major lines of research on the theoretical
understanding of neural networks. The first one deals with understanding
the properties of the objective function used when training neural
networks
\citep{ChoromanskaHenaffMathieuArousLeCun2014,SagunGuneyArousLeCun2014,ZhangLeeJordan2015,LivniShalevShwartzShamir2014,Kawaguchi2016}.
The second involves studying the black-box optimization algorithms
that are often used for training these networks
\citep{HardtRechtSinger2015,LianHuangLiLiu2015}.  The third analyzes
the statistical and generalization properties of the neural networks
\citep{Bartlett1998,ZhangeEtAl2016,NeyshaburTomiokaSrebro2015,SunChenWangLiu2016}.
The fourth takes the generative point of view
\citep{AroraBhaskaraGeMa2014, AroraLiangMa2015}, assuming that the
data actually comes from a particular network and then show how to
recover it. The fifth investigates the expressive ability of neural
networks and analyzing what types of mappings they can learn
\citep{CohenSharirShashua2015,EldanShamir2015,Telgarsky2016,DanielyFrostigSinger2016}.
This paper is most closely related to the work on statistical and
generalization properties of neural networks. However, instead of
analyzing the problem of learning with a fixed architecture we study a
more general task of learning both architecture and model parameters
simultaneously. On the other hand, the insights that we gain by
studying this more general setting can also be directly applied to the
setting with a fixed architecture.

There has also been extensive work involving structure learning for
neural networks
\citep{KwokYeung1997,LeungLamLingTam2003,IslamYaoMurase2003,Lehtokangas1999,IslamSattarAminYaoMurase2009,MaKhorasani2003,NarasimhaDelashmitManryLiMaldonado2008,HanQiao2013,KotaniKajikiAkazawa1997,AlvarezSalzmann2016}.
All these publications seek to grow and prune the neural network
architecture using some heuristic\ignore{ (e.g. genetic, information
  theoretic, or correlation)}. More recently, search-based approaches
have been an area of active research
\citep{HaDaiLe16,ChenGoodfellowShlens2015,ZophLe2016,BakerGuptaNaikRaskar2016}.
In this line of work, a learning meta-algorithm is used to search for
an efficient architecture.  Once better architecture is found
previously trained networks are discarded. This search requires a
significant amount of computational resources.  To the best of our
knowledge, none of these methods come with a theoretical guarantee on
their performance. Furthermore, optimization problems associated with
these methods are intractable.  In contrast, the structure learning
algorithms introduced in this paper are directly based on
data-dependent generalization bounds and aim to solve a convex
optimization problem by adaptively growing network and preserving
previously trained components.

Finally, \citep{JanzaminSedghiAnandkumar2015} is another paper that
analyzes the generalization and training of two-layer neural networks
through tensor methods. Our work uses different methods, applies to
arbitrary networks, and also learns a network structure from a single
input layer.

\section{Proofs}
\label{app:theory}

We will use the following structural learning guarantee for ensembles of hypotheses.

\begin{theorem}[DeepBoost Generalization Bound, Theorem~1, \citep{CortesMohriSyed2014}]
\label{lemma:db}
Let $\cH$ be a hypothesis set admitting a decomposition
$\cH = \cup_{i = 1}^l \cH_i$ for some $l > 1$. Fix $\rho > 0$. Then,
for any $\delta > 0$, with probability at least $1 - \delta$ over the
draw of a sample $S$ from $D^m$, the following inequality holds for
any $f = \sum_{t = 1}^T \alpha_t h_t$ with $\alpha_t \in \Rset_+$ and
$\sum_{t = 1}^T \alpha_t = 1$:
\begin{align*}
R(f) 
&\leq \widehat{R}_{S, \rho} + \frac{4}{\rho} \sum_{t=1}^T \alpha_t \mathfrak{R}_m(\mathcal{H}_{k_t}) + \frac{2}{\rho} \sqrt{\frac{\log l}{m}} \\
&\quad + \sqrt{\bigg\lceil \frac{4}{\rho^2} \log\left(\frac{\rho^2 m}{\log l}\right) \bigg\rceil \frac{\log l}{m} + \frac{\log(\frac{2}{\delta})}{2m}},  
\end{align*}
where, for each $h_t \in \cH$, $k_t$ denotes the smallest $k \in [l]$
such that $h_t \in \cH_{k_t}$.
\end{theorem}

\begin{reptheorem}{th:adanet}
Fix $\rho > 0$.  Then, for any $\delta > 0$, with probability at least
$1 - \delta$ over the draw of a sample $S$ of size $m$ from $\sD^m$,
the following inequality holds for all
$f = \sum_{k = 1}^l \bw_k \cdot \bh_k \in \sF$:
\begin{align*}
R(f) 
& \leq \widehat{R}_{S, \rho}(f) + \frac{4}{\rho} \sum_{k = 1}^{l} \big \| \bw
  _k \big \|_1 \R_m(\tl \sH_k) + \frac{2}{\rho} \sqrt{\frac{\log l}{m}}\\
& \quad + C(\rho, l, m, \delta),   
\end{align*}
where
$C(\rho, l, m, \delta) = \sqrt{\big\lceil \frac{4}{\rho^2}
  \log(\frac{\rho^2 m}{\log l}) \big\rceil \frac{\log l}{m} +
  \frac{\log(\frac{2}{\delta})}{2m}} = \tl O\Big(\frac{1}{\rho}
\sqrt{\frac{\log l}{m}}\Big)$.
\end{reptheorem}

\begin{proof}
This result follows directly from Theorem~\ref{lemma:db}.
\end{proof}

Theorem~\ref{th:adanet} can be straightforwardly generalized to
the multi-class classification setting by using the ensemble
margin bounds of
\citet{KuznetsovMohriSyed2014}.

\begin{replemma}{lemma:Rad_Hk_Hk-1}
For any $k > 1$, the empirical Rademacher complexity of $\sH_k$ for a
sample $S$ of size $m$ can be upper-bounded as follows in terms of
those of $\sH_s$s with $s < k$:
\begin{equation*}
\h \R_S(\sH_k) 
\leq 2 \sum_{s=1}^{k-1} \Lambda_{k,s}
n_{s}^{\frac{1}{q}} \h \R_S(\sH_{s}).
\end{equation*}
\end{replemma}

\begin{proof}
By definition, $\h \R_S(\sH_k)$ can be expressed as follows:
\begin{equation*}
\h \R_S(\sH_k) 
  = \frac{1}{m} \E_\ssigma \mspace{-5mu} \left[ \sup_{\substack{\bh_s \in
      \sH_s^{n_s}\\\| \bu_s \|_p \leq \Lambda_{k, s}}} \mspace{-5mu} \sum_{i = 1}^m
  \sigma_i \sum_{s = 1}^{k - 1} \bu_s \cdot (\varphi_s \circ
  \bh_s)(x_i) \right] \mspace{-5mu}.
\end{equation*}
  By the sub-additivity of the supremum, it
  can be upper-bounded as follows:
\begin{equation*}
  \h \R_S(\sH_k) 
  \leq \sum_{s = 1}^{k - 1} \frac{1}{m} \E_\ssigma \mspace{-5mu} \left[ \sup_{\substack{\bh_s \in
      \sH_s^{n_s}\\\| \bu_s \|_p \leq \Lambda_{k, s}}} \sum_{i = 1}^m
  \sigma_i \bu_s \cdot (\varphi_s \circ \bh_s)(x_i) \right] \mspace{-5mu}.
\end{equation*}
  We now bound each term of this sum, starting with the following
  chain of equalities:
\begin{align*}
& \frac{1}{m} \E_\ssigma \mspace{-5mu} \left[ \sup_{\substack{\bh_s \in
      \sH_s^{n_s}\\\| \bu_s \|_p \leq \Lambda_{k, s}}} \sum_{i = 1}^m
  \sigma_i \bu_s \cdot (\varphi_s \circ \bh_s)(x_i) \right] \\
& = \frac{\Lambda_{k,s}}{m} \E_\ssigma\left[ \sup_{\bh_s \in
      \sH_s^{n_s}} \bigg\| \sum_{i = 1}^m
  \sigma_i (\varphi_s \circ \bh_s)(x_i) \bigg\|_q
  \right]  \\
& = \frac{\Lambda_{k,s} n_{s}^{\frac{1}{q}}}{m} \E_\ssigma\left[ \sup_{\substack{h \in
  \sH_{s}}} \bigg| \sum_{i = 1}^m \sigma_i (\varphi_{s} \circ h)(x_i) \bigg|
  \right] \\
  & = \frac{\Lambda_{k,s} n_{s}^{\frac{1}{q}} }{m} \E_\ssigma\left[ \sup_{\substack{h \in
  \sH_{s}\\ \sigma \in \set{-1, +1}}} \sigma \sum_{i = 1}^m \sigma_i (\varphi_{s} \circ h)(x_i) 
  \right],
\end{align*}
where the second equality holds by definition of the dual norm and the
third equality by the following equality:
\begin{align*}
\sup_{z_i \in Z} \| \bz \|_q 
& = \sup_{z_i \in Z} \Big[ \sum_{i = 1}^n |z_i|^q \Big]^{\frac{1}{q}}
= \Big[ \sum_{i = 1}^n [\sup_{z_i \in Z} |z_i|]^q \Big]^{\frac{1}{q}}\\
& = n^{\frac{1}{q}} \sup_{z_i \in Z} |z_i|.
\end{align*}
The following chain of inequalities concludes the
proof:
\begin{align*}
  &\frac{\Lambda_{k,s} n_{s}^{\frac{1}{q}} }{m} \E_\ssigma\left[ \sup_{\substack{h \in
  \sH_{s}\\ \sigma \in \set{-1, +1}}} \sigma \sum_{i = 1}^m \sigma_i (\varphi_{s} \circ h)(x_i) 
  \right] \\
          & \leq \frac{\Lambda_{k,s} n_{s}^{\frac{1}{q}} }{m} \E_\ssigma\left[ \sup_{\substack{h \in
  \sH_{s}}} \sum_{i = 1}^m \sigma_i (\varphi_{s} \circ h)(x_i) \right] \\
&\quad + \frac{\Lambda_{k,s}}{m} \E_\ssigma\left[ \sup_{\substack{h \in
  \sH_{s}}} \sum_{i = 1}^m -\sigma_i (\varphi_{j} \circ h)(x_i) 
  \right] \\
  & = \frac{2 \Lambda_{k,s} n_{s}^{\frac{1}{q}} }{m} \E_\ssigma\left[ \sup_{\substack{h \in
  \sH_{s}}} \sum_{i = 1}^m \sigma_i (\varphi_{s} \circ h)(x_i) \right] \\
  & \leq \frac{2 \Lambda_{k,s} n_{s}^{\frac{1}{q}} }{m} \E_\ssigma\left[ \sup_{\substack{h \in
  \sH_{s}}} \sum_{i = 1}^m \sigma_i h(x_i) \right] \\
                                                             & \leq 2 \Lambda_{k,s} n_{s}^{\frac{1}{q}} \h \R_S(\sH_{s}),
\end{align*}
where the second inequality holds by Talagrand's contraction lemma.
\end{proof}

\begin{replemma}{lemma:Rad_Hk}
Let $\Lambda_k \!= \prod_{s = 1}^k 2 \Lambda_{s, s - 1}$ and
$N_k \!= \prod_{s = 1}^k n_{s - 1}$. Then, for any $k \geq 1$, the
empirical Rademacher complexity of $\sH_k^*$ for a sample $S$ of size
$m$ can be upper bounded as follows:
\begin{equation*}
\h \R_S(\sH_k^*) 
\leq r_\infty \Lambda_k N_k^{\frac{1}{q}} 
\sqrt{\frac{\log (2 n_0)}{2 m}}.
\end{equation*}
\end{replemma}

\begin{proof}
The empirical Rademacher complexity of $\sH_1$ can be bounded as follows:
\begin{align*}
\h \R_S(\sH_1) 
& = \frac{1}{m} \E_\ssigma\left[ \sup_{\| \bu \|_p \leq \Lambda_{1,0}} \sum_{i = 1}^m \sigma_i \bu \cdot \bPsi(x_i) \right]\\
& = \frac{1}{m} \E_\ssigma\left[ \sup_{\| \bu \|_p \leq \Lambda_{1,0}}
  \bu \cdot \sum_{i = 1}^m \sigma_i \bPsi(x_i)  \right] \\
& = \frac{\Lambda_{1,0}}{m} \E_\ssigma\left[ \bigg\| \sum_{i = 1}^m \sigma_i
  [\bPsi(x_i)] \bigg\|_q \right]  \ignore{& (\text{def. of dual norm})}\\
& \leq \frac{\Lambda_{1,0} n_0^{\frac{1}{q}}}{m} \E_\ssigma\left[ \bigg\| \sum_{i = 1}^m \sigma_i
  [\bPsi(x_i)] \bigg\|_\infty \right]  \ignore{& (\text{equivalence of $l_p$ norms})}\\
& = \frac{\Lambda_{1,0} n_0^{\frac{1}{q}}}{m} \E_\ssigma\left[ \max_{j \in [1, n_1]} \bigg | \sum_{i = 1}^m \sigma_i
  [\bPsi(x_i)]_j \bigg | \right]  \ignore{& (\text{def. of $l_\infty$ norm})}\\
& = \frac{\Lambda_{1,0} n_0^{\frac{1}{q}}}{m} \E_\ssigma\left[ \max_{\substack{j \in [1,
  n_1]\\ s \in \set{-1, +1}}} \sum_{i = 1}^m \sigma_i
  s[\bPsi(x_i)]_j  \right]  \ignore{& (\text{def. of absolute value})}\\
& \leq \Lambda_{1,0}  n_0^{\frac{1}{q}} r_\infty \sqrt{m} \frac{\sqrt{2 \log (2 n_0)}}{m} \\
&= r_\infty \Lambda_{1,0} n_0^{\frac{1}{q}}\sqrt{\frac{2 \log (2 n_0)}{m}}. \ignore{& (\text{Massart's lemma})}
\end{align*}
The result then follows by application of Lemma~\ref{lemma:Rad_Hk_Hk-1}.
\end{proof}

\begin{repcorollary}{cor:adanet}
Fix $\rho > 0$.  Let
$\Lambda_k = \prod_{s = 1}^k 4 \Lambda_{s, s - 1}$ and
$N_k \!= \prod_{s = 1}^k n_{s - 1}$. Then, for any $\delta > 0$, with
probability at least $1 - \delta$ over the draw of a sample $S$ of
size $m$ from $\sD^m$, the following inequality holds for all
$f = \sum_{k = 1}^l \bw_k \cdot \bh_k \in \sF^*$:
\begin{align*}
R(f) 
& \leq \widehat{R}_{S, \rho}(f) + \frac{2}{\rho} \sum_{k = 1}^{l} \big \| \bw
  _k \big \|_1 \bigg[\ov r_\infty \Lambda_k N_k^{\frac{1}{q}} \sqrt{\frac{2 \log (2 n_0)}{m}}\bigg] \\
& \quad + \frac{2}{\rho} \sqrt{\frac{\log l}{m}} + C(\rho, l, m, \delta),   
\end{align*}
where $C(\rho, l, m, \delta) = \sqrt{\big\lceil \frac{4}{\rho^2}
  \log(\frac{\rho^2 m}{\log l}) \big\rceil \frac{\log l}{m} +
  \frac{\log(\frac{2}{\delta})}{2m}} = \tl O\Big(\frac{1}{\rho}
\sqrt{\frac{\log l}{m}}\Big)$, and where $\ov r_\infty = \E_{S \sim \sD^m}[r_\infty]$.
\end{repcorollary}

\begin{proof}
  Since $\sF^*$ is the convex hull of $\sH^*$, we can apply
  Theorem~\ref{th:adanet} with $\R_m(\tl \sH^*_k)$ instead of
  $\R_m(\tl \sH_k)$.  Observe that, since for any $k \in [l]$,
  $\tl \sH^*_k$ is the union of $\sH^*_k$ and its reflection, to
  derive a bound on $\R_m(\tl \sH^*_k)$ from a bound on
  $\R_m(\tl \sH_k)$ it suffices to double each $\Lambda_{s, s -
    1}$. Combining this observation with the bound of
  Lemma~\ref{lemma:Rad_Hk} completes the proof.
\end{proof}

\section{Alternative Algorithm}
\label{sec:alternatives}

In this section, we present an alternative algorithm, \textsc{AdaNet.CVX}, that 
generates candidate subnetworks in closed-form using Banach space duality.

As in Section~\ref{sec:algorithm}, let $f_{t-1}$ denote the \textsc{AdaNet}
model after $t-1$ rounds, and let $l_{t-1}$ be the depth of the architecture. 
\textsc{AdaNet.CVX} will consider $l_{t-1} + 1$ candidate subnetworks, one
for each layer in the model plus an additional one for extending the model.

Let $h^{(s)}$ denote the candidate subnetwork associated to layer $s \in [l_{t-1}+1]$. 
We define $h^{(s)}$ to be a single unit in layer $s$ that is connected
to units of $f_{t-1}$ in layer $s-1$:
\begin{align*}
  &h^{(s)} \in \{x \mapsto \bu \cdot (\varphi_{s-1} \circ \bh_{s-1,t-1})(x) : \\
  &\quad \bu \in \Rset^{n_{s-1,t-1}}, \|\bu\|_p \leq \Lambda_{s,s-1}\}.
\end{align*}
See Figure~\ref{fig:adanetcvx} for an illustration of the type of neural network
designed using these candidate subnetworks.

\begin{figure}[t] 
\centering
  \includegraphics[scale=.35]{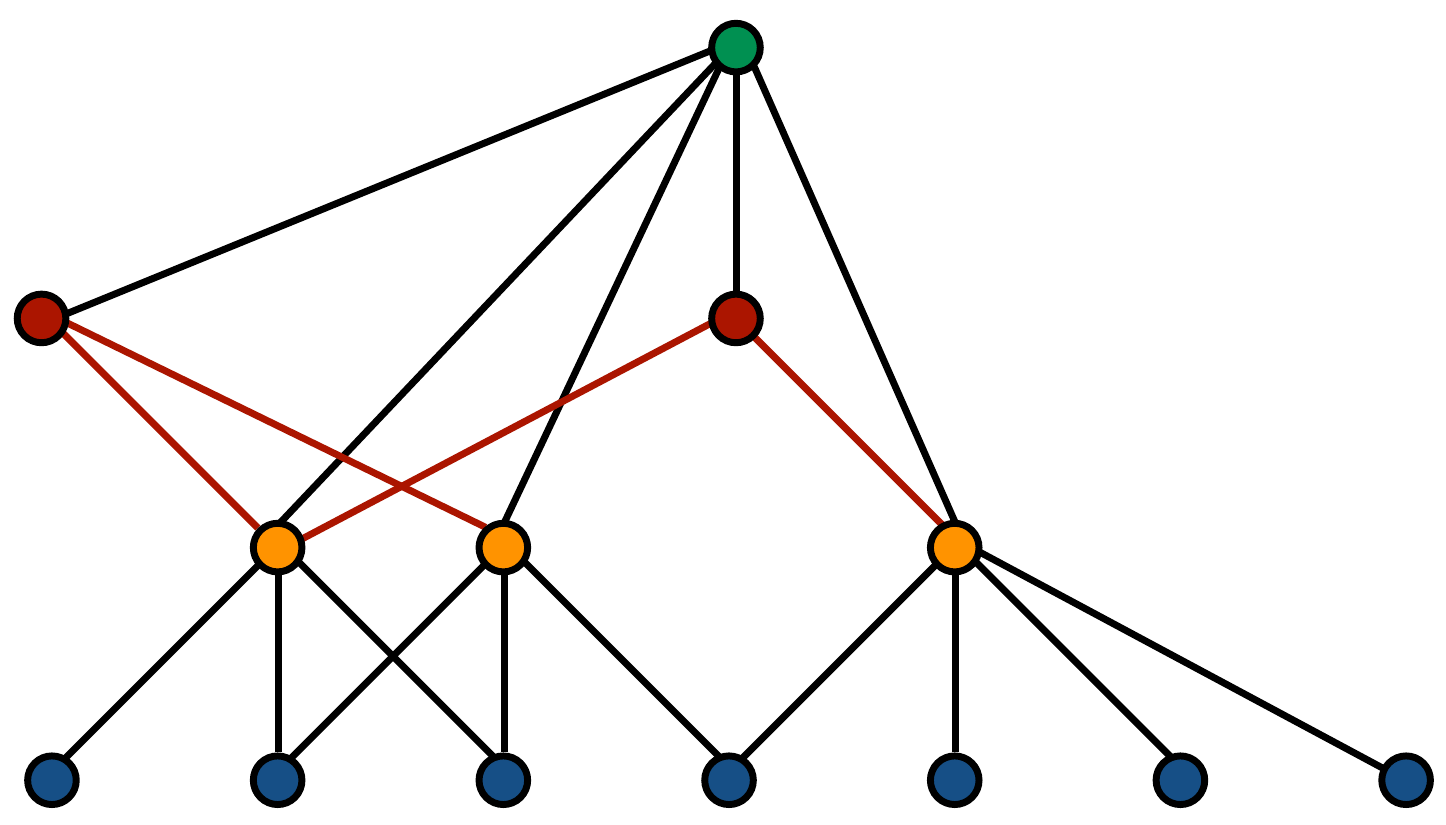}
  \caption{Illustration of a neural network designed by \textsc{AdaNet.CVX}.
  Units at each layer (other than the output layer) are only connected
  to units in the layer below.}
\vskip -.1in
\label{fig:adanetcvx}
\end{figure}

For convenience, we denote this space of subnetworks by $\sH_s'$: 
\begin{align*}
  &\sH_{s}' =  \{x \mapsto \bu \cdot (\varphi_{s-1} \circ \bh_{s-1,t-1})(x) : \\
  &\quad \bu \in \Rset^{n_{s-1,t-1}}, \|\bu\|_p \leq \Lambda_{s,s-1}\}.
\end{align*}

Now recall the notation 
\begin{align*}
  &F_t(w, h) \\
  &= \frac{1}{m} \sum_{i=1}^m \Phi\Big(1 - y_i (f_{t-1}(x_i) - w h(x_i))\Big) + \Gamma_{\bh} \|\bw\|_1
\end{align*}
used in Section~\ref{sec:algorithm}. As in \textsc{AdaNet},  
the candidate subnetwork chosen by \textsc{AdaNet.CVX} is given
by the following optimization problem:
\begin{align*}
  \argmin_{h \in \cup_{s=1}^{l_{t-1}+1} \sH_{s}'} \min_{w \in \Rset} F_t(w, h).
\end{align*}
Remarkably, the subnetwork that solves this infinite dimensional optimization 
problem can be obtained directly in closed-form: 

\begin{theorem}[\textsc{AdaNet.CVX} Optimization]
  \label{th:adanetcvx}
  Let $(w^*, h^*)$ be the solution to the following optimization problem:
\begin{align*}
  \argmin_{h \in \cup_{s=1}^{l_{t-1}} \sH_{s}'} \min_{w \in \Rset} F_t(w, h).
\end{align*}
  Let $D_t$ be a distribution over the sample $(x_i,y_i)_{i=1}^m$ such that $D_t(i) \propto \Phi'(1 - y_i f_{t-1}(x_i)$, and denote
  $\eps_{t,h} = \E_{i \sim D_t}[y_ih(x_i)]$. 
  
  Then, 
  $$w^* h^* = w^{(s^*)}h^{(s^*)},$$ 
  where $(w^{(s^*)}, h^{(s^*)})$ are defined by:
  \begin{align*}
    &s^* =  \argmax_{s\in[l_t-1]} \Lambda_{s,s-1} \|\be_{t,\bh_{s-1,t-1}}\|_q. \\
    &\bu^{(s^*)} = u^{(s)}_i = \frac{\Lambda_{s,s-1}}{\|\be_{t,\bh_{s-1,t-1}}\|_{q}^{\frac{q}{p}}} |\e_{t,h_{s-1,t-1,i}}|^{q-1} \\
    &h^{(s^*)} = \bu^{(s^*)} \cdot (\varphi_s \circ \bh_{s-1,t-1}) \\
    &w^{(s^*)} = \argmin_{w \in \Rset}\frac{1}{m} \sum_{i=1}^m \Phi\Big(1 - y_i f_{t-1}(x_i) \\ 
    &\quad - y_iw h^{(s^*)}(x_i)\Big) + \Gamma_{s^*} |w|.
  \end{align*}
\end{theorem}
\begin{proof}
  By definition,
  \begin{align*}
    &F_t(w, h)\\ 
    &= \frac{1}{m} \sum_{i=1}^m \Phi\Big(1 - y_i (f_{t-1}(x_i) - w h(x_i))\Big) + \Gamma_{h} |w|.
  \end{align*}
   
  Notice that the  
  minimizer over $\cup_{s=1}^{l_{t-1}+1} \sH_{s}'$ can be determined by
  comparing the minimizers over each $\sH_s'$. 
  
  Moreover, since the penalty term $\Gamma_h |w|$ has the same contribution for every 
  $h \in \sH_{s}'$, it has no impact on the optimal choice of $h$ over $\sH_s'$. Thus, to find the minimizer 
  over each $\sH_{s}'$, we can compute the derivative of $F_t - \Gamma_h |w|$ with respect to
  $w$:
  \begin{align*}
    &\frac{d(F_t - \Gamma_h |\eta|)}{dw}(w, h) \\ 
    &= \frac{-1}{m} \sum_{i=1}^m y_i h(x_i) \Phi'\Big( 1 - y_i f_{t-1}(x_i)  \Big) .
  \end{align*}
  Now if we let 
  $$D_t(i) S_t = \Phi'\Big( 1 - y_i f_{t-1}(x_i) \Big),$$
  then this expression is equal to
  \begin{align*}
    &-\left[ \sum_{i=1}^m y_i h(x_i) D_t(i) \right] \frac{S_t}{m} 
    = (2\e_{t,h} - 1) \frac{S_t}{m},
  \end{align*}
  where $\e_{t,h} = \E_{i \sim D_t} [y_i h(x_i)]$.

  Thus, it follows that for any $s \in [l_{t-1} + 1]$,
  \begin{align*}
    \argmax_{h \in \sH_s'} \frac{d(F_t - \Gamma_h |w|)}{dw}(w, h)
    &= \argmax_{h \in \sH_s'} \e_{t,h}.
  \end{align*}

  Note that we still need to search for the optimal descent coordinate 
  over an infinite dimensional space. However, we can write 
  \begin{align*}
    &\max_{h \in \sH_s'} \be_{t,h } \\
    &= \max_{h \in \sH_s'}\E_{i \sim D_t} [y_i h(x_i)] \\
    &= \max_{\bu \in \Rset^{n_{s-1,t-1}}} \E_{i \sim D_t} [y_i \bu \cdot (\varphi_{s-1} \circ \bh_{s-1,t-1})(x_i)] \\
    &= \max_{\bu \in \Rset^{n_{s-1,t-1}}} \bu \cdot \E_{i \sim D_t} [y_i \cdot (\varphi_{s-1} \circ \bh_{s-1,t-1})(x_i)]. 
  \end{align*}

  Now, if we denote by $u^{(s)}$ the connection weights associated to $h^{(s)}$, then we claim that
  \begin{align*}
    u^{(s)}_i = \frac{\Lambda_{s,s-1}}{\|\be_{t,\bh_{s-1,t-1}}\|_{q}^{\frac{q}{p}}} |\e_{t,h_{s-1,t-1,i}}|^{q-1},
  \end{align*}
  which is a consequence of Banach space duality.

  To see this, note first that by H\"{o}lder's inequality, every $\bu \in \Rset^{n_{s-1,t-1}}$ with $\|\bu\|_p \leq \Lambda_{s,s-1}$ satisfies:
  \begin{align*}
    &\bu \cdot \E_{i \sim D_t} [y_i \cdot (\varphi_{s-1} \circ \bh_{s-1,t-1})(x_i)]\\
    &\leq \|\bu\|_p \|\E_{i \sim D_t} [y_i \cdot (\varphi_{s-1} \circ \bh_{s-1,t-1})(x_i)]\|_q\\
    &\leq \Lambda_{s,s-1} \|\E_{i \sim D_t} [y_i \cdot (\varphi_{s-1} \circ \bh_{s-1,t-1})(x_i)]\|_q.
  \end{align*}
  At the same time, our choice of $\bu^{(s)}$ also attains this upper bound:
  \begin{align*}
    &\bu^{(s)} \cdot \be_{t,\bh_{s-1,t-1}} \\
    &= \sum_{i=1}^{n_{s-1,t-1}} u^{(s)}_i \e_{t,h_{s-1,t-1,i}} \\
    &= \sum_{i=1}^{n_{s-1,t-1}} \frac{\Lambda_{s,s-1}}{\|\be_{t,\bh_{s-1,t-1}}\|_q^{\frac{q}{p}}} |\e_{t,h_{s-1,t-1,i}}|^q \\
    &= \frac{\Lambda_{s,s-1}}{\|\be_{t,\bh_{s-1,t-1}}\|_q^{\frac{q}{p}}}\|\be_{t,\bh_{s-1,t-1}}\|_q^q \\
    &= \Lambda_{s,s-1} \|\be_{t,\bh_{s-1,t-1}}\|_q.
  \end{align*}
  Thus, $\bu^{(s)}$ and the associated network $h^{(s)}$ is the coordinate that maximizes the derivative of $F_t$ with respect to $w$ among all
  subnetworks in $\sH_s'$. Moreover, $h^{(s)}$ also achieves the value: $\Lambda_{s,s-1} \|\be_{t,\bh_{s-1,t-1}}\|_q$. 

  This implies that by computing $\Lambda_{s,s-1} \|\be_{t,\bh_{s-1,t-1}}\|_q$ for every $s \in [l_{t-1}+1]$, we can find
  the descent coordinate across all $s \in [l_{t-1}+1]$ that improves the objective by the largest amount. Moreover, we can then
  solve for the optimal step size in this direction to compute the weight update. 

\end{proof}

The theorem above defines the choice of descent coordinate at each round
and motivates the following algorithm, \textsc{AdaNet.CVX}. At each round,
\textsc{AdaNet.CVX} can design the optimal candidate subnetwork within
its searched space in closed form, leading to an extremely efficient
update.
However, this comes at the cost of a more restrictive search space
than the one used in \textsc{AdaNet}.
The pseudocode of \textsc{AdaNet.CVX} is provided in Figure~\ref{algo:adanetcvx}. 

\begin{figure}[t]
  \begin{ALGO}{AdaNet.CVX}{S = ((x_i, y_i)_{i = 1}^m}
  \SET{f_0}{0}
\DOFOR{t \EQ 1 \TO T}
    \SET{s^*}{\argmax_{s\in[l_{t-1}+1]} \Lambda_{s,s-1} \|\be_{t,\bh_{s-1,t-1}}\|_q.}
    \SET{u^{(s^*)}_i}{\frac{\Lambda_{s^*,s^*-1}}{\|\be_{t,\bh_{s^*-1,t-1}}\|_{q}^{\frac{q}{p}}} |\e_{t,h_{s^*-1,t-1,i}}|^{q-1}} 
    \SET{\bh'}{\bu^{(s^*)} \cdot (\phi_{s^*-1} \circ \bh_{s^*-1,t-1})}
    \SET{\eta'}{\textsc{Minimize}(\tilde{F}_t(\eta, \bh'))} 
    \SET{f_t}{f_{t-1} + \eta' \cdot \bh'}
\OD
\RETURN{f_T}
\end{ALGO}
\vskip -.5cm
\caption{Pseudocode of the AdaNet.CVX algorithm.}
\vskip -.0cm
\label{algo:adanetcvx}
\end{figure}

\end{document}